\documentclass[a4,11pt]{article}

\usepackage{times}
\usepackage{mathptmx}
\usepackage{latexsym}
\usepackage{amsmath} 
\usepackage{amssymb}  

\usepackage{dsfont} 

\usepackage{url} 
\usepackage{epsf}
\usepackage{xspace}	
\usepackage{pgf}

\usepackage{amsthm} 
\newtheorem{theorem}{Theorem}
\newtheorem{proposition}[theorem]{Proposition}
\newtheorem{corollary}[theorem]{Corollary}
\newtheorem{lemma}[theorem]{Lemma}

\usepackage{amsopn}

\DeclareMathOperator*{\argmin}{arg\,min}

\newcommand{\m}[1]{\ensuremath{\overline{#1}}}

\definecolor{red}{rgb}{0.8,0.2,0.2}
\definecolor{blue}{rgb}{0,0,0.5}
\definecolor{green}{rgb}{0,0.7,0}
\definecolor{violet}{rgb}{0.5,0.2,0.5}
\definecolor{orange}{rgb}{0.8,0.5,0.2}

\newcommand{\red}[1][red]{\textcolor{red}{{#1}}}

\newcommand{\blue}[1][blue]{{#1}}
\newcommand{\green}[1][green]{{#1}}

\newcommand{\xaxis}{$x$-axis\xspace}
\newcommand{\yaxis}{$y$-axis\xspace}
\newcommand{\auc}{\ensuremath{\mathit{AUC}}\xspace}
\newcommand{\auch}{\ensuremath{\mathit{AUCH}}\xspace}
\newcommand{\bs}{\ensuremath{\mathit{BS}}\xspace}

\newcommand{\sk}{z}

\newcommand{\Tcost}{T_c}
\newcommand{\Tsk}{T_\sk}

\newcommand{\Tcosto}{{T^o_c}}
\newcommand{\Tcostp}{{T^p_c}}

\newcommand{\Tcostq}{{T^q_c}}

\newcommand{\Tcostn}{{T^n_c}}

\newcommand{\Tsko}{{T^o_\sk}}
\newcommand{\Tskp}{{T^p_\sk}}

\newcommand{\Qcost}{Q_c}
\newcommand{\Qsk}{Q_\sk}

\newcommand{\wcost}{w_c}
\newcommand{\wsk}{w_\sk}

\newcommand{\Lcost}{L_c}
\newcommand{\Lsk}{L_\sk}

\newcommand{\Lsko}{L^o_\sk}

\newcommand{\LskoU}{L^o_{U(\sk)}}

\newcommand{\Lcostn}{L^n_c}

\newcommand{\LcostnU}{L^n_{U(c)}}
\newcommand{\LsknU}{L^n_{U(\sk)}}

\newcommand{\Lcostp}{L^p_c}
\newcommand{\LcostpU}{L^p_{U(c)}}

\newcommand{\CCsk}{CC_{\sk}}
\newcommand{\CCsko}{CC^o_{\sk}}

\newcommand{\CCcostn}{CC^n_{c}}

\newcommand{\rate}{r}  
\newcommand{\rcurve}{ROC cost curve\xspace}  
\newcommand{\rcurves}{ROC cost curves\xspace}  
\newcommand{\briercurve}{Brier curve\xspace}  
\newcommand{\briercurves}{Brier curves\xspace}  

\title{ROC Curves in Cost Space}
\author{Peter A.~Flach \and Jos\'e Hern\'andez-Orallo \and C\`esar Ferri\footnote{\red[Order of authors to be determined.]} \\
\red[INSTITUTIONS...]
}

\begin{document}


\begin{titlepage}

\begin{center}


\vspace*{1cm}

\newcommand{\HRule}{\rule{\linewidth}{0.5mm}}

\HRule \\[0.4cm]
{ \huge \bfseries Technical Note: Towards ROC Curves in Cost Space}\\[0.4cm]

\HRule \\[3cm]

\begin{minipage}{0.75\textwidth}
\begin{flushleft} 
Jos\'e Hern\'andez-Orallo\hfill(jorallo@dsic.upv.es)\\
Departament de Sistemes Inform\`atics i Computaci\'o\\
Universitat Polit\`ecnica de Val\`encia, Spain\\[12pt]

Peter Flach\hfill(Peter.Flach@bristol.ac.uk)\\
Intelligent Systems Laboratory\\
University of Bristol, United Kingdom\\[12pt]

C\`esar Ferri\hfill(cferri@dsic.upv.es)\\
Departament de Sistemes Inform\`atics i Computaci\'o\\
Universitat Polit\`ecnica de Val\`encia, Spain\\[12pt]

\end{flushleft}
\end{minipage}

\vspace{2cm}

{\large \today}

\vfill
\end{center}

\begin{abstract}
\noindent
ROC curves and cost curves are two popular ways of visualising classifier performance, finding appropriate thresholds according to the operating condition, and deriving useful aggregated measures such as the area under the ROC curve (\auc) or the area under the optimal cost curve. 
In this note we present some new findings and connections between ROC space and cost space, by using the expected loss over a range of operating conditions. In particular, we show that ROC curves
can be transferred to cost space by means of a very natural way of understanding how thresholds should be chosen, by selecting the threshold such that the proportion of positive predictions equals the operating condition (either in the form of cost proportion or skew).
We call these new curves {\em \rcurves}, and we demonstrate that the expected loss as measured by the area under these curves is linearly related to \auc. This opens up a series of new possibilities and clarifies the notion of cost curve and its relation to ROC analysis.
In addition, we show that for a classifier that assigns the scores in an evenly-spaced way, these curves are equal to the \briercurves. As a result, this establishes the first clear connection between $\auc$ and the Brier score.
\\[12pt]

\end{abstract}

\end{titlepage}

\section{Introduction}

There are many graphical representations and tools for classifier evaluation, such as ROC curves \cite{SDM00,Fawcett06}, ROC isometrics \cite{Fla03}, cost curves \cite{DH00,drummond-and-Holte2006}, DET curves \cite{martin1997det}, lift charts \cite{piatetsky1999estimating}, calibration maps \cite{cohen2004properties}, among others. 
In this paper, we will focus on ROC curves and cost curves. These are often considered to be two sides of the same coin, where a point in ROC space corresponds to a line in cost space. However, this is only true up to a point, as a curve in ROC space has no corresponding representation in cost space. It is true that the {\em convex hull} of a ROC curve corresponds to the lower envelope of all the cost lines, but this is not the ROC curve. In fact, the area under this lower envelope has no clear connection with $\auc$. As a result, cost space cannot be used in the same way as ROC spaces, and we find some advantages in one representation over the other and vice versa.

One of the issues with this lack of full correspondence is that the definition of what a cost curve is has been rather vague in the literature. In some occasions, only the cost lines are formally defined \cite{drummond-and-Holte2006}, where the {\em curve} is just defined as the lower envelope of all these lines.
 However, this assumes that threshold choices are optimal, which is not generally the case.
This curve is what we call here `the optimal cost curve' (frequently referred to in the literature as `{\em the} cost curve').
It is worth mentioning that Drummond and Holte \cite{drummond-and-Holte2006} talk about `selection criteria' (instead of `threshold choice methods'), they distinguish between `performance-independent selection criteria' and `cost-minimizing selection criteria', and they show some curves using different `selection criteria'. However, they do not develop the ideas further and they do not use this to generalise the notion of cost curve.

In previous work, we have generalised and systematically developed the concept of threshold choice method. For instance, in \cite{ICML11CoherentAUC} we have explored a new instance-uniform threshold choice method while in \cite{ICML11Brier} we explore the probabilistic threshold method. 
In \cite{thresholdchoicepaper11} we analyse this in general, leading to a total of six threshold choice methods and its corresponding measures.

In this paper, we are interested in how all this can be plotted in cost space and, in particular, we analyse a new threshold choice method which assigns the threshold such that the proportion or rate of positive predictions equals the operating condition (cost proportion). This leads to a cost curve where all the segments have equal length in terms of its projection over the \xaxis. In other words, each segment covers a range of cost proportions of equal length. 

A first graphical analysis of this curve indicates that each segment corresponds to a point in ROC space, and its position with respect to the optimal cost curve gives virtually the same information as the ROC curve does. Consequently, we call this new curve {\em the \rcurve}. It can also be interpreted as a cost-based analysis of rankers. 
Further analysis of this curve shows that the area under the \rcurve is a linear function of $\auc$, so doubly justifying the name given to this curve, its interpretation and its applications.

The paper is organised as follows. Section \ref{notation} introduces some basic notation and definitions.
In Section \ref{optimal} we refer to the relation between the ROC convex hull and the optimal cost curve.
Section \ref{ROCcost} introduces one of the contributions in the paper by using a threshold choice method which leads to the \rcurves. It also explains how these curves can be plotted easily and what their interpretation is. 
Section \ref{area} shows that the area under this curve is a linear function of $\auc$, and demonstrates the correspondence for some typical cases (random classifier, perfect classifier, worst classifier).
Section \ref{Brier} analyses when a classifier chooses its scores in an evenly-spaced way. In this case,
it turns out that the area under the \rcurve is exactly the Brier score.
Section \ref{conclusion} closes the paper with some conclusions and future work.

\section{Notation and basic definitions}\label{notation}

In this section we introduce some basic notation and the notions of ROC curve, cost curves and the way expected loss is aggregated using a threshold-choice method.
Most of this section is reused from \cite{thresholdchoicepaper11}.

\subsection{Notation}


We denote by $U_S(x)$ the continuous uniform distribution of variable $x$ over an interval $S \subset {\mathds{R}}$. If this interval $S$ is $[0,1]$ then $S$ can be omitted. 

Examples (also called instances) are taken from an instance space.
\blue[The instance space is denoted $X$ and the output space $Y$. 
Elements in $X$ and $Y$ will be referred to as $x$ and $y$ respectively. For this paper we will assume binary classifiers, i.e., $Y = \{0, 1\}$. A crisp or categorical classifier is a function that maps examples to classes.] 
A probabilistic classifier is a function $m:X \rightarrow [0,1]$ that maps examples to estimates $\hat{p}(1|x)$ of the probability of example $x$ to be of class $1$.
\blue[  
A scoring classifier is a function $m:X \rightarrow \Re$ that maps examples to real numbers on an unspecified scale, such that scores are monotonically related to $\hat{p}(1|x)$. 
]
\blue[In order to make predictions in the $Y$ domain, a probabilistic or scoring classifier can be converted to a crisp classifier by fixing a decision threshold $t$ on the scores. Given a predicted score $s=m(x)$, the instance $x$ is classified in class $1$ if $s > t$, and in class $0$ otherwise. 
]

\blue[For a given, unspecified classifier and population from which data are drawn, we denote the score density for class $k$ by $f_k$ and the cumulative distribution function by $F_k$. Thus, 
$F_{0}(t) = \int_{-\infty}^{t} f_{0}(s) ds = P(s\leq t|0)$ 
is the proportion of class 0 points correctly classified if the decision threshold is $t$, which is the sensitivity or true positive rate at $t$. Similarly, 
$F_{1}(t) = \int_{-\infty}^{t} f_{1}(s) ds = P(s\leq t|1)$ 
is the proportion of class 1 points incorrectly classified as 0 or the false positive rate at threshold $t$; $1-F_{1}(t)$ is the true negative rate or specificity.
\footnote{We use 0 for the positive class and 1 for the negative class, but scores increase with $\hat{p}(1|x)$. That is, a ranking from strongest positive prediction to strongest negative prediction has non-decreasing scores. This is the same convention as used by, e.g., \cite{hand2009measuring}. }
]           

\blue[
Given a data set $D \subset \langle X, Y\rangle$ of size $n = |D|$, we denote by $D_k$ the subset of examples in class $k \in \{0,1\}$,
and set $n_k = |D_k|$ and $\pi_k = n_k / n$. We will use the term \emph{class proportion} for $\pi_0$ (other terms such as `class ratio' or `class prior' have been used in the literature).]
The average score of class $k$ is
$\m{s}_k = {1 \over n_k} \sum_{\langle x, y \rangle \in {D_k}} m(x)$.
\blue[Given any strict order for a data set of $n$ examples we will use the index $i$ on that order to refer to the $i$-th example. Thus, $s_i$ denotes the score of the $i$-th example and $y_i$ its true class.]  
We use $I$ to denote the set of indices, i.e. $I=1..n$.
\blue[Given a data set and a classifier, we can define empirical score distributions for which we will use the same symbols as the population functions. We then have $f_k(s)={1 \over n_{k}}|\{\langle x, y \rangle \in {D_k} | m(x)=s \}|$ 
which is non-zero only in $n'_{k}$ points, where $n'_{k} \leq n_{k}$ is the number of unique scores assigned to instances in $D_{k}$ (when there are no ties, we have $n_k' = n_k$). 
Furthermore, the cumulative distribution functions 
and $F_{k}(t) = \sum_{s\leq t} f_{k}(s)$  
are piecewise constant with $n'_{k}+1$ segments].

$F_0$ is called sensitivity and $F_1$ is called specificity. 
The meaning of $F_0(t)$ can be seen as the proportion of examples of class $0$ which are correctly classified if the threshold is set at $t$.
Conversely, the meaning of $1-F_1(t)$ can be seen as the proportion of examples of class $1$ which are correctly classified if the threshold is set at $t$.  


\subsection{Operating conditions and overall loss}

When a classification model is applied, the conditions or context might be different to those used during its training might. In fact, a classifier can be used in several contexts, with different results. A context can imply different class proportions, different cost over examples (either for the attributes, for the class or any other kind of cost), or some other details about the effects that the application of a model might entail and the severity of its errors.
In practice, \blue[an operating condition or deployment context is usually defined by ] a misclassification cost function and a class distribution. 
\blue[Clearly, there is a difference between operating when the cost of misclassifying $0$ into $1$ is equal to the cost of misclassifying $1$ into $0$ and doing so when the former is ten times the latter. Similarly, operating when classes are balanced is different from when there is an overwhelming majority of instances of one class.
]%

\blue[One general approach to cost-sensitive learning assumes that the cost does not depend on the example but only on its class. In this way, misclassification costs are usually simplified by means of cost matrices, where we can express that some misclassification costs are higher than others \cite{Elk01}. Typically, the costs of correct classifications are assumed to be 0.]%
\footnote{Not doing so, or just considering one of the correct classifications to have 0 cost will lead to results which are different to the simplified setting by a constant term or factor, as happens with the model for cost-loss ratio used by Murphy in \cite{murphy1966note}.}
\blue[This means that for binary classifiers we can describe the cost matrix by two values $c_k \geq 0$, representing the misclassification cost of an example of class $k$. Additionally, we can normalise the costs by setting $b = c_0 + c_1$ and $c = c_0 / b$; 
we will refer to $c$ as the \emph{cost proportion}.] 
Since this can also be expressed as $c=(1 + c_1/c_0)^{-1}$, it is often called `cost ratio' even though, technically, it is a proportion ranging between $0$ and $1$. 
We can see the dependency between $b$ and $c$\footnote{Hand \cite[p115]{hand2009measuring}
assumes  $b$ and $c$ to be independent, and hence considers $b$ not necessarily a constant. However, in the end, he also assumes that the result is only affected by a constant factor.}, which leaves just one degree of freedom, and we can set one of them constant.
Consequently, choosing $b$ constant we see that it \blue[ only affects the magnitude of the costs but is independent of the classifier. We set $b=2$ so that loss is commensurate with error rate (which just assumes $c_{0}=c_{1}=1$). 
]

\blue[The loss which is produced at a decision threshold $t$ and a cost proportion $c$ is then given by the formula:]
\begin{align}\label{eqQcost}
\Qcost(t; c) & \triangleq c_{0} \pi_0 (1 -F_0(t)) + c_{1} \pi_1 F_1(t) \\
 & = 2\{c \pi_0 (1 -F_0(t)) + (1-c) \pi_1 F_1(t)\} \nonumber
\end{align}
%
\blue[We often are interested in analysing the influence of class proportion and cost proportion at the same time.
Since the relevance of $c_0$ increases with $\pi_{0}$, an appropriate way to consider both at the same time is by the definition of \emph{skew}, which is a normalisation of their product:]
\begin{align}\label{eqSkew}
\sk & \triangleq \frac{c_0\pi_0}{c_0\pi_0 + c_1\pi_1} = \frac{c\pi_0}{c\pi_0 + (1-c)(1-\pi_0)}
\end{align}
It follows that $c = \frac{z\pi_1}{z\pi_1 + (1-z)(1-\pi_1)}$.
\blue[From Eq.~(\ref{eqQcost}) we obtain]
\begin{align}\label{eqQsk}
\frac{\Qcost(t; c)}{c_0\pi_0 + c_1\pi_1} = \sk (1 -F_0(t)) + (1-\sk) F_1(t) & \triangleq \Qsk(t; \sk)
\end{align}
\blue[This gives an expression for loss at a threshold $t$ and a skew $\sk$. 
We will assume that the operating condition is either defined by the cost proportion (using a fixed class distribution) or by the skew.]
\green[We then have the following simple but useful result]

\begin{lemma}\label{lemma-balance}
\green[If $\pi_0 = \pi_1$ then $z=c$ and $\Qsk(t; \sk) = {2 \over b} \Qcost(t; c)$. 
]
\end{lemma}
\vspace{-0.5cm}
\begin{proof}
\green[If classes are balanced we have $c_0\pi_0 + c_1\pi_1 = b/2$, and the result follows from Eq.~(\ref{eqSkew}) and Eq.~(\ref{eqQsk}).] 
\end{proof}

\green[This further justifies taking $b=2$, which means that $\Qsk$ and $\Qcost$ are expressed on the same 0-1 scale, and, as said above, are also commensurate with error rate which assumes $c_{0}=c_{1}=1$. 
The upshot of Lemma~\ref{lemma-balance} is that we can transfer any expression for loss in terms of cost proportion to an equivalent expression in terms of skew by just setting $\pi_0 = \pi_1=1/2$ and $z=c$. ]



In many real problems, when we have to evaluate or compare classifiers, we do not know the cost proportion or skew that will apply during application time. One general approach is to evaluate the classifier on a range of possible operating points. In order to do this, we have to set a weight or distribution on cost proportions or skews. In this paper, we will consider the continuous uniform distribution $U$.

A key issue when applying a classifier to several operating conditions is how the threshold is chosen in each of them. If we work with a crisp classifier, this question vanishes, since the threshold is already settled. However, in the general case when we work with a soft probabilistic classifier, we have to decide how to establish the threshold. The crucial idea explored in this paper is the notion of {\em threshold choice method}, a function $T(c)$ or $T(\sk)$, which converts an operating condition (cost proportion or skew) into an appropriate threshold for the classifier. There are several reasonable options for the function $T$. We can set a fixed threshold for all operating conditions, we can set the threshold by looking at the ROC curve (or its convex hull) and using the cost proportion or the skew to intersect the ROC curve (as ROC analysis does), we can set a threshold looking at the estimated scores, especially when they represent probabilities, or we can set a threshold independently from the rank or the scores. The way in which we set the threshold may dramatically affect performance. But, not less importantly, the performance measure used for evaluation must be in accordance with the threshold choice method.

From this interpretation, Adams and Hand \cite{AdamsHand1999} suggest to set a distribution over the set of possible operating points and integrate over them. In this way, we can define the overall or average expected loss in a range of situations as follows:
\begin{equation}\label{eqLcost}
\Lcost \triangleq \int^1_0{} \Qcost(\Tcost(c); c) \wcost(c) dc 
\end{equation}
where $\Qcost(t)$ is the expected cost for threshold $t$ as seen above, $\Tcost$ is a threshold choice method, which maps cost proportions to thresholds, and $\wcost(c)$ is a distribution for costs in $[0,1]$. Clearly we see that any performance measure which attempts to measure average expected cost in a wide range of operating condition depends on two things. First, the distribution $\wcost(c)$ that we use to weight the range of conditions. Second, the threshold choice method $T_c$. 
Additionally, we can define this overall or average expected cost to be independent of the class priors, so defining a similar construction for skews instead of costs:

\begin{equation}\label{eqLsk}
\Lsk \triangleq \int^1_0{} \Qsk(\Tsk(\sk); \sk) \wsk(\sk) d\sk
\end{equation}

If we draw $\Qcost$ or $\Qsk$ over $c$ and $\sk$ respectively, we get a plot space known as cost plots or curves, as we will illustrate below. Cost curves are also known as risk curves (see, e.g. \cite{reid2011information}, where the plot can also be shown in terms of {\em priors}, i.e. class proportions).

So a cost curve as a function of $\sk$ in our notation is simply:
\begin{equation}\label{eq:CCsk}
\CCsk(z) \triangleq \Qsk(T(\sk); \sk)
\end{equation}

\noindent and similarly for cost proportions. Note that it is the threshold choice method $T$ which can draw a different curve for the same classifier.

\subsection{Some common plots and measures}

In what follows, we introduce some common evaluation measures: the Brier Score, the ROC space and the Area Under the ROC curve ($\auc$). In the following section we also introduce the convex hull and the optimal cost curves.

\blue[
The Brier score is a well-known evaluation measure for probabilistic classifiers. It is an alternative name for the Mean Squared Error or MSE loss \cite{brier1950verification}, especially for binary classification. 
$\bs(m,D)$ is the Brier score of classifier $m$ with data $D$; we will usually omit $m$ and $D$ when clear from the context. We define $\bs_k(m,D) = \bs(m,D_k)$. \bs is defined as follows:
]
\begin{equation}\label{eqBS}
\bs \triangleq    {1 \over n} \sum_{i=1}^{n}(s_i - y_i)^2 = \pi_0 \bs_0 + \pi_1 \bs_1
\end{equation}
\blue[where $s_i$ is the score predicted for example $i$ and $y_i$ is the true class for example $i$. 
The corresponding population quantities are 
$\bs_{0} = \int_{0}^{1} s^{2}f_{0}(s) ds$ and 
$\bs_{1} = \int_{0}^{1} (1-s)^{2}f_{1}(s) ds$.  
]

%
%
%

\blue[The ROC curve \cite{SDM00,Fawcett06} is defined as a plot of $F_1(t)$ (i.e., false positive rate at decision threshold $t$) on the \xaxis against $F_0(t)$ (true positive rate at $t$) on the \yaxis, with both quantities monotonically non-decreasing with increasing $t$ (remember that scores increase with $\hat{p}(1|x)$ and 1 stands for the negative class).]
Figure \ref{fig:example1} (Leftmost: dash lines) shows a ROC curve for a classifier  with 4 examples of class 1 and 11 examples of class 0. Because of ties, there are 11 distinct scores and hence 11 bins/segements in the ROC curve.

From a ROC curve, we can derive the Area Under the ROC curve (\auc) as:
\begin{eqnarray}
\auc & \triangleq & \int_{0}^{1} F_0(s) d F_1(s) = \int_{-\infty}^{+\infty} F_0(s) f_1(s) ds = \int_{-\infty}^{+\infty} \int_{-\infty}^{s} f_0(t) f_1(s) dt ds  \label{eq:AUC}\\
     & = & \int_{0}^{1} (1-F_1(s)) d F_0(s) = \int_{-\infty}^{+\infty} (1-F_1(s)) f_0(s) ds = \int_{-\infty}^{+\infty} \int_{s}^{+\infty} f_1(t) f_0(s) dt ds \nonumber
\end{eqnarray}
\blue[When dealing with empirical distributions the integral is replaced by a sum. ]

\section{The optimal cost curve}\label{optimal}

Given a scoring (or soft) classifier, one approach for choosing a classification threshold is to consider that (1) we are having complete information about the operating condition (class proportions and costs) and (2) we are able to use that information to choose the threshold that will minimise the cost using the current classifier.
ROC analysis is precisely based on these two points and, as we have seen, using the skew and the convex hull, we can calculate the threshold which gives the smallest loss (for the training set).

This threshold choice method, denoted by $\Tcosto$ is:
\begin{align}\label{eqTcosto}
\Tcosto(c) & \triangleq   \argmin_{t}\ \{\Qcost(t; c)\}\nonumber  \\
& =  \argmin_{t}\ 2\{c\pi_0(1-F_0(t)) + (1-c)\pi_1 F_1(t)\}
\end{align}
\blue[
which matches the {optimal} threshold for a given skew $\sk$: 
]
%
\begin{align} \label{eqTsko}
\Tsko(\sk) & \triangleq \argmin_{t}\ \{\Qsk(t; \sk)\} 
           = \Tcosto(c)
\end{align}
This threshold gives the convex hull in the ROC space.
\blue[The convex hull of a ROC curve (ROCCH) is a construction over the ROC curve in such a way that all the points on the ROCCH have minimum loss for some choice of $c$ or $\sk$. 
This means that we restrict attention to the {\em optimal} threshold for a given cost proportion $c$.]
Note that the $\argmin$ will typically give a range (interval) of values which give the same optimal value.
\blue[The convex hull is defined by the points $\{F_1(t), F_0(t)\}$ where $t = \Tcosto(c)$ for some $c$.] Then, in order to make a hull, all the remaining points are linearly interpolated (pairwise). 
All this is shown in Figure \ref{fig:example1} (leftmost).
\blue[The Area Under the ROCCH (denoted by \auch) can be computed in a similar way as the \auc with modified versions of $f_k$ and $F_k$. Obviously, $\auch \ge \auc$, with equality implying the ROC curve is convex.]

\blue[
A cost plot as defined by \cite{drummond-and-Holte2006} has $\Qsk(t; \sk)$ on the \yaxis against skew $\sk$ on the \xaxis (Drummond and Holte use the term `probability cost' rather than skew). Since 
$\Qsk(t; \sk) = \sk (1 -F_0(t)) + (1-\sk) F_1(t)$, 
cost lines for a given decision threshold $t$ are straight lines $\Qsk = a_0 + a_1\sk$ with intercept $a_0=F_{1}(t)$ and slope $a_1=1-F_{0}(t)-F_{1}(t)$.
]
\blue[A cost line visualises how cost at that threshold changes between $F_{1}(t)$ for $z=0$ and $1-F_{0}(t)$ for $z=1$.]

From all the set of cost lines, we can choose line segments and by piecewise connecting them we have a `hybrid cost curve' \cite{drummond-and-Holte2006}. One way of choosing these segments is by considering the optimal threshold. Hence, the optimal or minimum cost curve \blue[is then the lower envelope of all the cost lines, obtained by only considering the optimal threshold (the lowest cost line) for each skew. 
The cost curve for this optimal choice is just given by instantiating equation (\ref{eq:CCsk}) with the optimal threshold choice method. Namely, for skews, we would have:
\begin{equation}\label{eq:CCsko}
\CCsko(z) \triangleq \Qsk(\Tsko(\sk); \sk)
\end{equation}
]
%
%

%
\begin{figure*}
\centering
\includegraphics[width=0.325\textwidth]{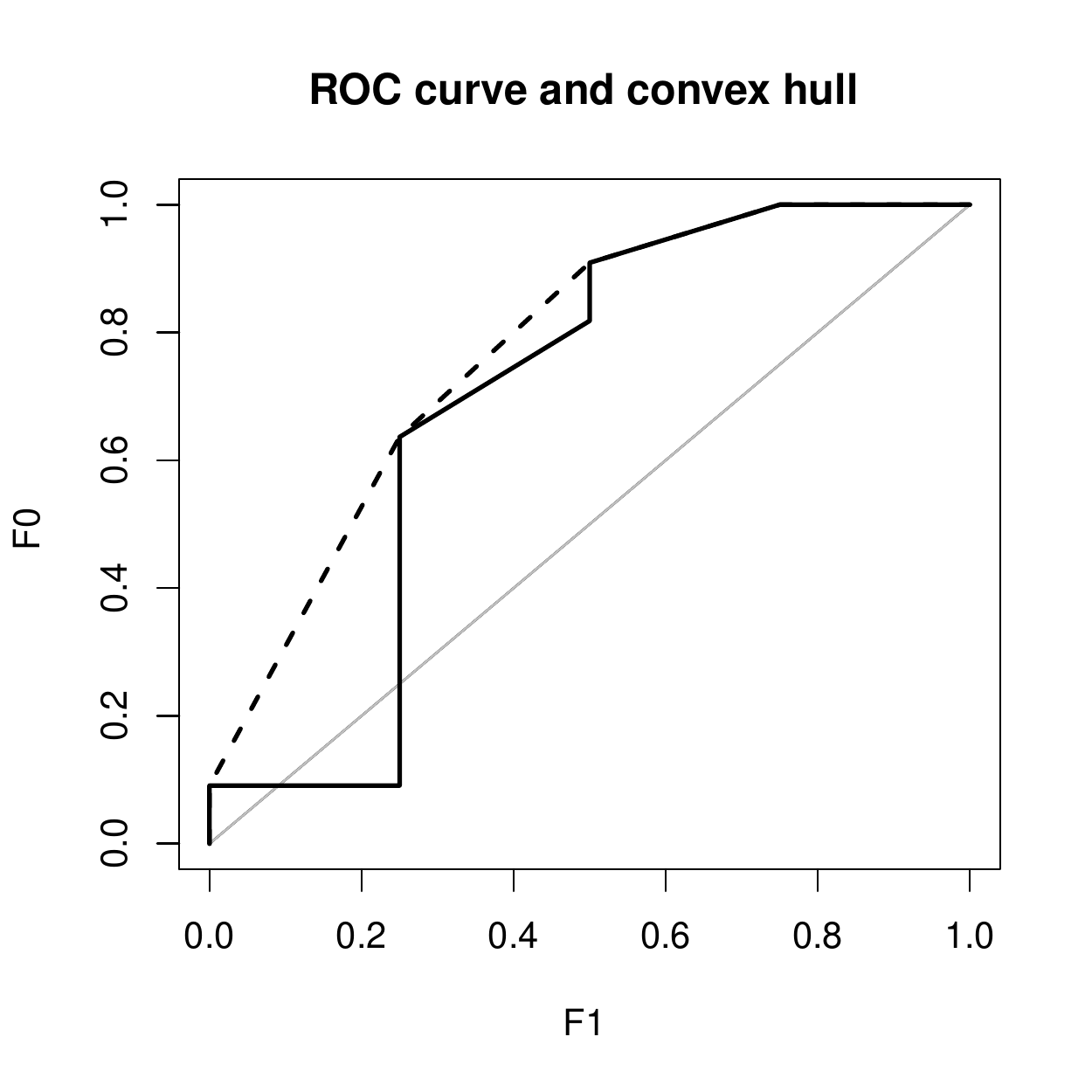}
\includegraphics[width=0.325\textwidth]{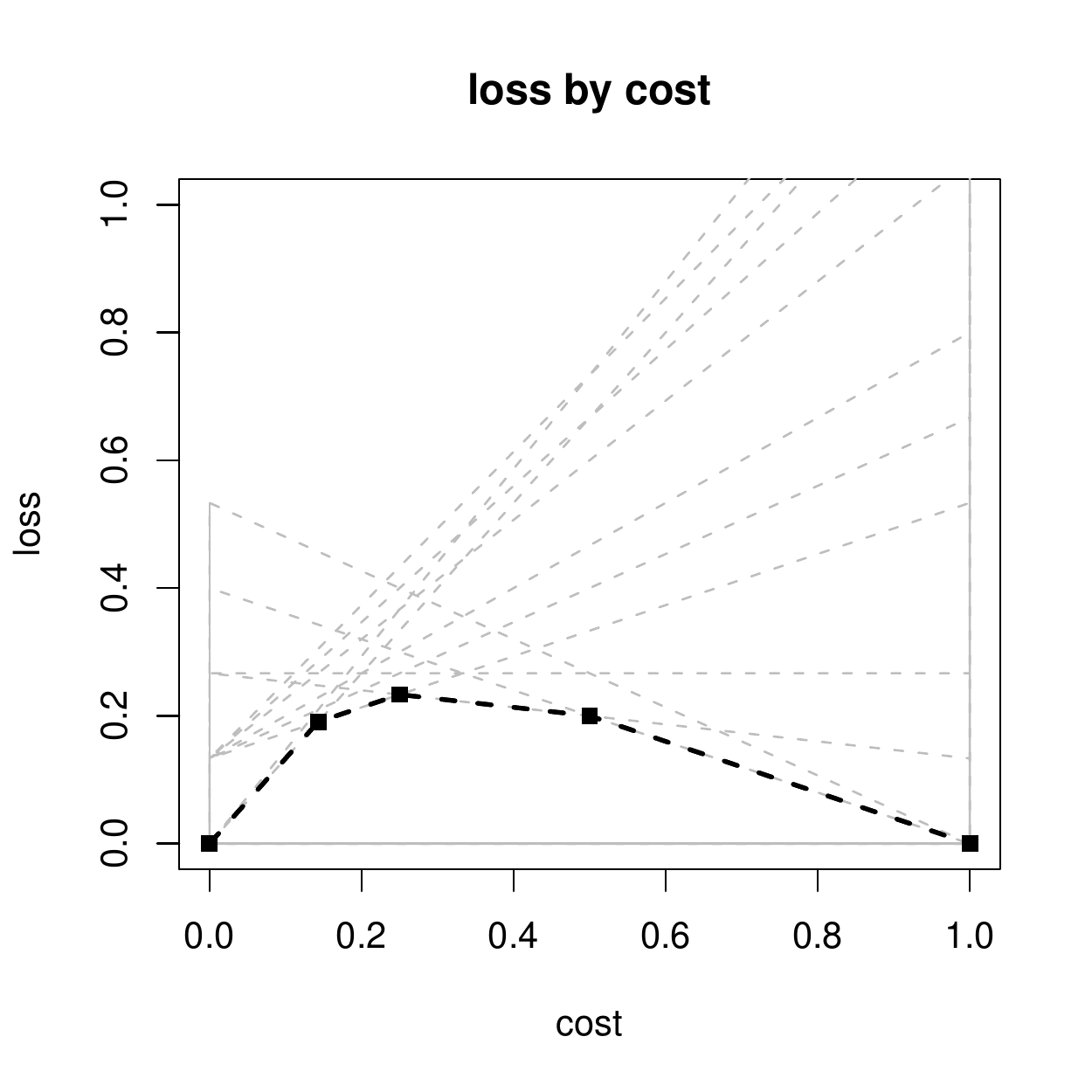}
\includegraphics[width=0.325\textwidth]{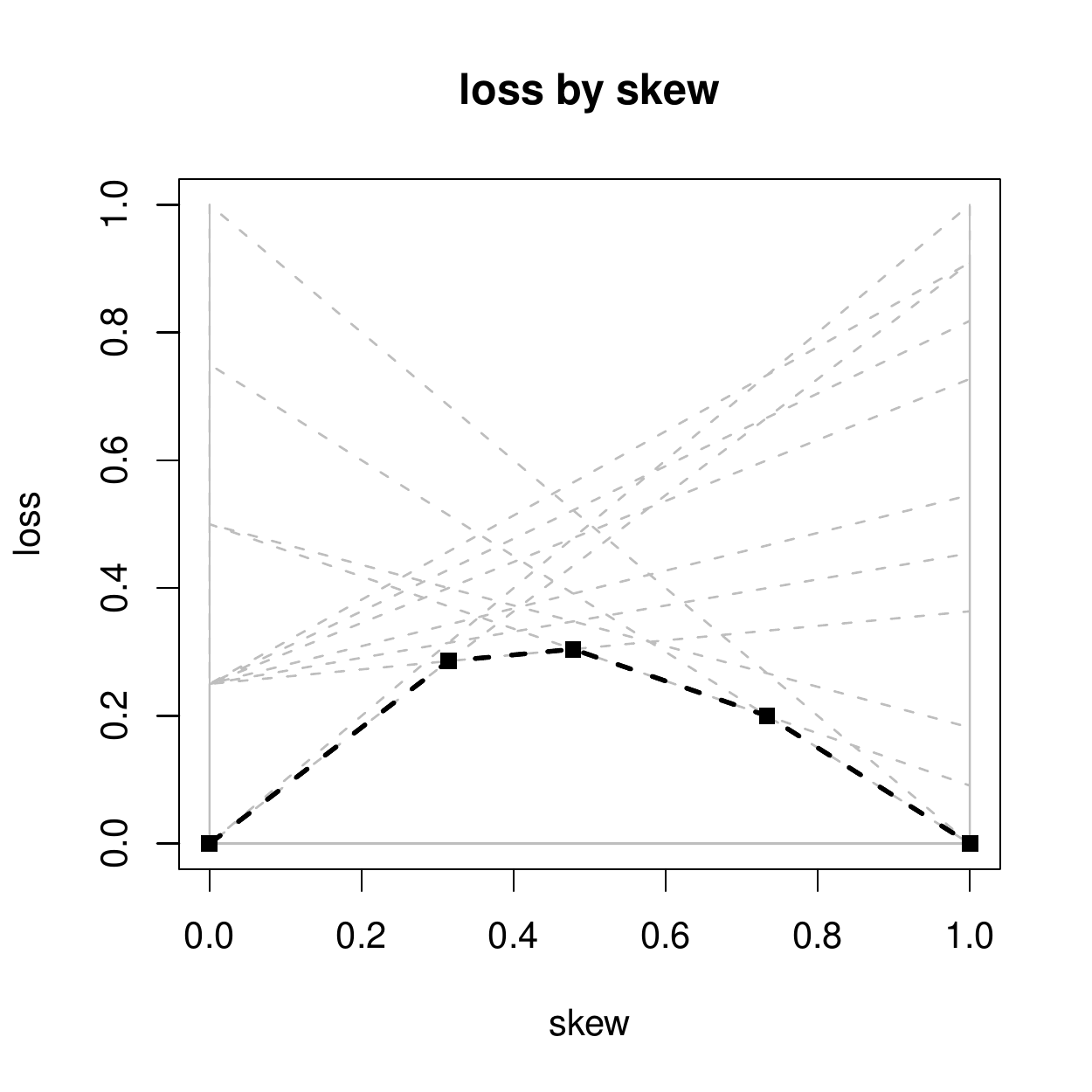}
\caption{\blue[Several graphical representations for the classifier] 
 with probability estimates (0.95, 0.90, 0.90, 0.85, 0.70, 0.70, 0.70, 0.55, 0.45, 0.20, 0.20, 0.18, 0.16, 0.15, 0.05)
 and classes (1, 0, 1, 0, 0, 1, 0, 0, 0, 0, 0, 0, 0, 1, 0)
\blue[Left: ROC curve] (solid) \blue[and convex hull] (dashed). Middle: cost lines and optimal cost curve against cost proportions. Right: cost lines and optimal cost curve against skews.}
\label{fig:example1}
\end{figure*}

Following the classifier and the ROC curve shown in Figure~\ref{fig:example1} (leftmost), we also show the optimal cost curve (rightmost) for that classifier. \blue[We observe 7 segments in the original ROC curve on the left, and 5 segments in its convex hull. We see that these 5 segments correspond to the 5 points in the optimal cost curve on the right. The optimal cost curve is `constructed' as the lower envelope of the 12 cost lines (one more than the number of distinct scores).

The middle plot in Figure~\ref{fig:example1} is an alternative cost plot with cost proportion rather than skew on the \xaxis. That is, here the cost lines are straight lines $\Qcost = a'_0 + a'_1c$ with intercept $a'_0=2\pi_{1}F_{1}(t)$ and slope $a'_1=2\pi_{0}(1-F_{0}(t))-2\pi_{1}F_{1}(t)$. We can clearly observe the class imbalance. 
]


For the classifier shown in Figure~\ref{fig:example1}, if we are given an extreme skew= 0.8, we know
that any threshold between 0.90 and 0.95 will be optimal, since it will classify example 15 as negative (1) and the rest as positive (0). This cutpoint (e.g. $t= 0.92 \in [0.90,0.95]$) gives $F_0=11/11$ and $F_1=3/4$, and minimises the loss for this skew, as given by Eq.\ (\ref{eqQsk}), i.e. $\Qsk(0.92, 0.8) = 0.8*(1-11/11) + (1-0.8)*(3/4) = 0.15$. Another cutpoint, e.g. $t= 0.85$, gives $F_0=10/11$ and $F_1=2/4$, with a higher $\Qsk(0.85, 0.8) = 0.8*(1-10/11) + (1-0.8)*(2/4) = 0.17$.

We may be interested in calculating the area under this optimal cost curve. If we use skews, we can derive:
\begin{equation}\label{eqLsko}
\Lsko \triangleq \int^1_0{} \Qsk(\Tsko(\sk); \sk) \wsk(\sk) d\sk
\end{equation}
But this equation is {\em exactly} the TEC (from `Total Expected Cost') given by Drummond and Holte (\cite{drummond-and-Holte2006} page 106, bottom). Drummond and Holte use the term `probability times cost' for skew (or simply, and somewhat misleadingly, `probability cost'). The distribution of probability costs is denoted by $prob(x)$ ($\wsk(\sk)$ in our notation). For $prob(x)$, Drummond and Holte choose the uniform distribution, i.e.:
\begin{equation}\label{eqLskoU}
\LskoU \triangleq \int^1_0{} \Qsk(\Tsko(\sk); \sk) U(\sk) d\sk 
\end{equation}
This expression is just the area under the cost curve. In Drummond and Holte's words: ``The area under a cost curve is the expected cost of the classifier assuming all possible probability cost values are equally likely, i.e. that $prob(x)$ is the uniform distribution." ($prob(x)$ is $\wsk(\sk)$ in our notation).

The problem of all this is that we are not always given all the information about the operating condition. In fact, even having that information, there are perfect techniques (namely ROC analysis) to get the optimal threshold for a data set (e.g. the training or validation data set), but this does not ensure that these choices are going to be optimal for a test set.
Consequently, evaluating classifiers in this way is a strong assumption. Additionally, how close the estimated optimal threshold is to the actual optimal threshold may depend on the classifier as well. One option is to consider confidence bands, but another option is just to drop this assumption.

\section{The \rcurve}\label{ROCcost}

The easiest way to choose the threshold is to set it  {\em independently} from the classifier and also from the operating condition. This mechanism can set the threshold in an absolute or a relative way. The absolute way, as explored in \cite{thresholdchoicepaper11}, just sets $T(c) = t$ (or, for skews, $T(\sk)= t$), with $t$ being a fixed threshold.
A simple variant of the fixed threshold is to consider that it is not
the absolute value of the threshold which is fixed, but a relative rate or proportion $\rate$ over the data set. In other words, this method tries to {\em quantify} the number of positive examples given by the threshold.
For example, we could say that our threshold is fixed to predict 30\% positives and the rest negatives.
This of course involves ranking the examples by their scores and setting a threshold at the appropriate position.
We will develop this idea for cost proportions below.

\subsection{The \rcurve for cost proportions}

The definition of of the \emph{rate-fixed} threshold choice method for costs is as follows:
\begin{equation}\label{eqTq2}
\Tcostq[\rate](c) 
\triangleq  \{ t : P(s_i < t) = \rate \} 
\end{equation}
In other words, we choose the threshold such that the probability that a score is lower than the threshold -- i.e., the positive prediction rate, is $\rate$. In the example in Figure \ref{fig:example1}, any value in the interval $[0.3,0.2]$ makes that the probability (or proportion) of the score being lower than that value is $2/6= 0.33$, which approximates $\rate = 0.3$.

It is interesting to connect the previous expression of this threshold given by Eq.\ (\ref{eqTq2}) with the cumulative distributions.
\begin{lemma}\label{lemmanew}
\begin{equation}
\Tcostq[\rate](c) = \{ t: F_0(t)\pi_0 + F_1(t)\pi_1 = \rate \}
\end{equation}
\end{lemma}
\begin{proof}
We can rewrite:
\[
P(s < t) = P(s < t|0)P(0) + P(s < t|1)P(1)
\]
But using the definition of $P(s < t|0)$ and $P(s < t|1)$ in the preliminaries in terms of the cumulative distributions, we have:
\[
P(s < t) = F_0(t)P(0) + F_1(t)P(1)=  F_0(t)\pi_0 + F_1(t)\pi_1
\]
so substituting into Eq.\ \ref{eqTq2} we have the result:
\[
\Tcostq[\rate](c) =  \{ t: F_0(t)\pi_0 + F_1(t)\pi_1 = \rate \}
\]
\end{proof}
\noindent
This straightforward result shows that this criterion clearly depends on the classifier, but it only takes the ranks into account, not the magnitudes of the scores.

However, there is a natural way of setting the positive prediction rate in an adaptive way. Instead of fixing the proportion of positive predictions, we
may take the operating condition into account.
If we have an operating condition, we can use the information about the skew or cost proportion to adjust the positive prediction rate to that proportion.
This leads to the \emph{rate-driven} threshold selection method: if we are given cost proportion $c$, we choose the threshold $t$ in such a way that we get a proportion of $c$ positive predictions.
\begin{equation}\label{eqTn}
\Tcostn(c) 
\triangleq  \Tcostq[c](c) 
= \{ t : P(s_i < t) = c \}  
\end{equation}
And given this threshold selection method, we can now derive its cost curve:
\begin{equation}\label{eq:CCcostn}
\CCcostn(c) \triangleq \Qcost(\Tcostn(c); c)
\end{equation}
Because of Lemma \ref{lemmanew}, we can see that this is equivalent to:
\begin{equation}
\CCcostn(c) = \Qcost(\{ t: F_0(t)\pi_0 + F_1(t)\pi_1 = c \}; c)
\end{equation}
Assuming no ties, we see that the expression $F_0(t)\pi_0 + F_1(t)\pi_1$ only changes its value between scores.
If have $n$ examples, it only changes $n+1$ times. So for finite populations, this has to be rewritten as follows:
\begin{equation}
\CCcostn(c) = \Qcost(\{ t: c-\frac{1}{n+1} < F_0(t)\pi_0 + F_1(t)\pi_1 \leq c \}; c)
\end{equation}
This leads to $n+1$ intervals in cost space where the threshold is not changed in each of these intervals. This means that the cost line is the same. This leads to the following procedure:

$\:$

\noindent {\bf \rcurve for cost proportions}: $\CCcostn$ \\
Given a classifier and a data set with $n$ examples:
\begin{enumerate}
\item Draw the $n+1$ cost lines, $CL_0$ to $CL_n$.
\item From left to right, draw the curve following each cost line (from $CL_0$ to $CL_n$) with a width on the $\xaxis$ of $\frac{1}{n+1}$.
\end{enumerate}

\begin{figure*}
\centering
\includegraphics[width=0.325\textwidth]{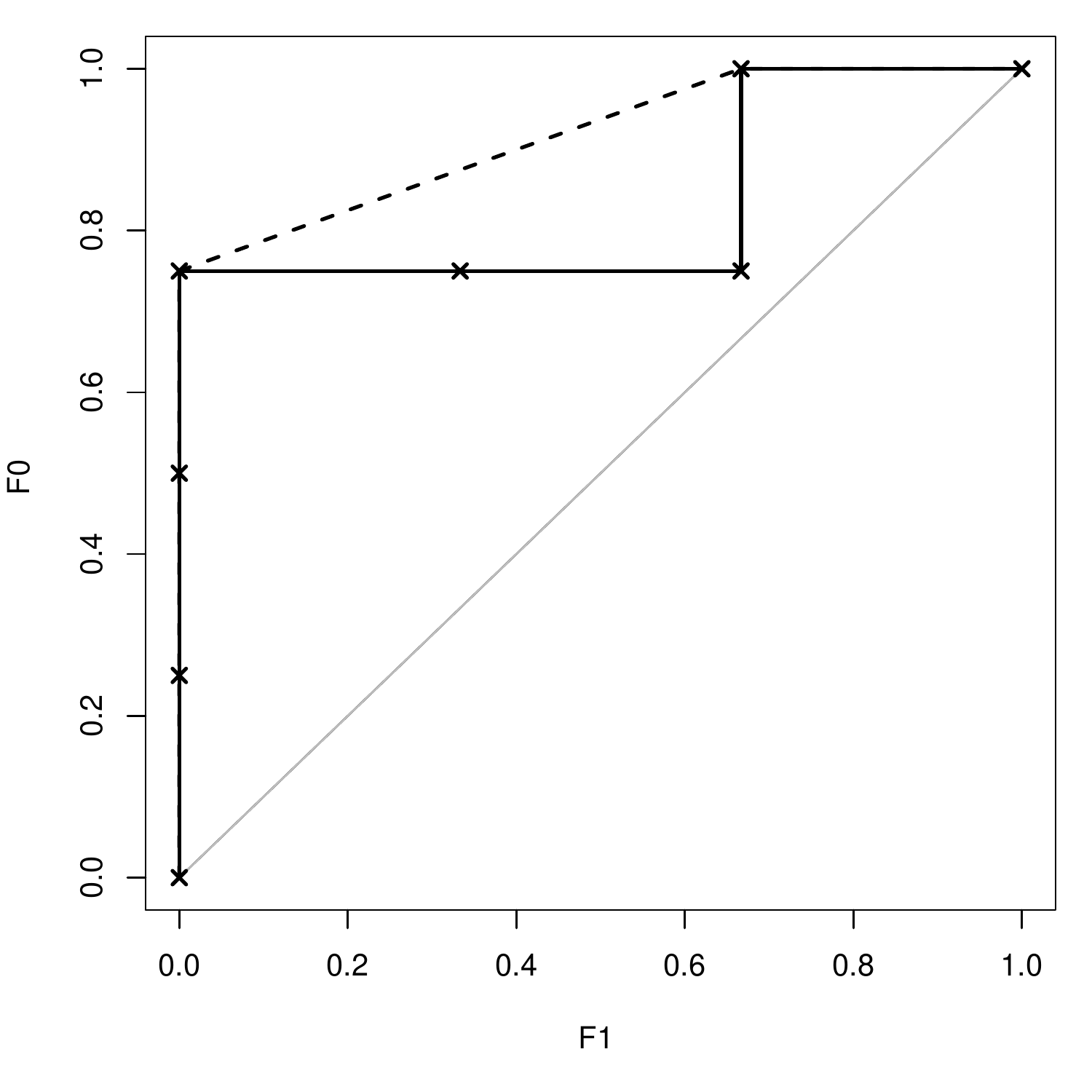}
\includegraphics[width=0.325\textwidth]{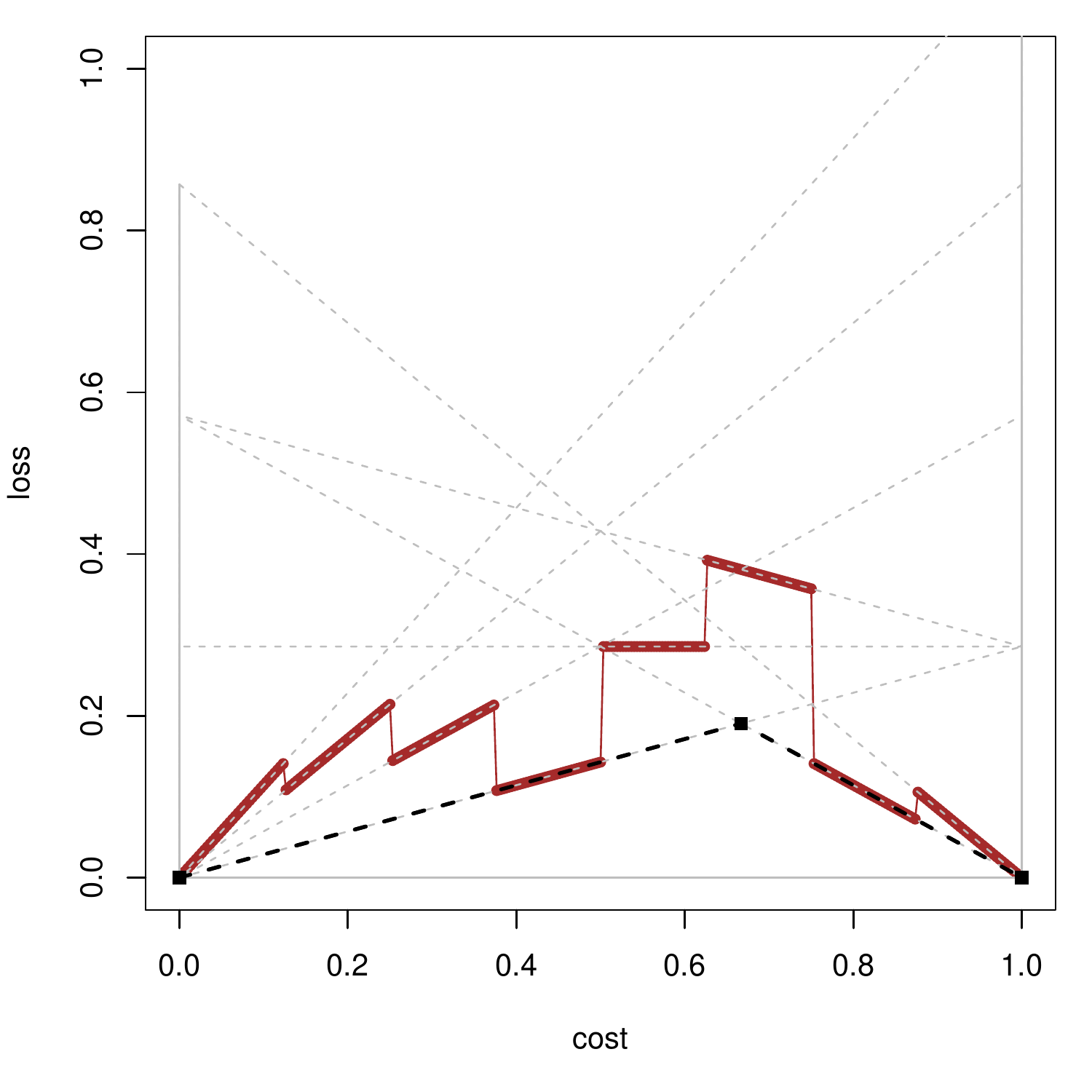}
\caption{Several graphical representations for the classifier 
 with probability estimates (0.95,0.9,0.8,0.3,0.2, 0.1,0.05)
 and classes (1,0,1,1,0,0,0). 
Left: ROC curve (solid) and convex hull (dashed). 
Right: cost lines, optimal cost curve (dashed) and \rcurve (thick and solid) against cost proportions. 
}
\label{fig:example2}
\end{figure*}

Figure \ref{fig:example2} shows a small classifier for a data set with 3 positive examples and 4 negative examples.
The ROC curve on the left has 8 points, since there are 8 cut points to choose the threshold, leading to 8 crisp classifiers and, accordingly, 8 cost lines. These cost lines are shown in the cost space in the plot on the right. 
We see that the projection of each segment onto the \xaxis has exactly a length of 1/8. Note that each segment uses a portion of each cost line.

It is relatively easy to understand what these curves mean and to see their correspondence to ROC curves.
Following Figure \ref{fig:example2}, going from (0,0) to (1,1) in the ROC curve, the first three points are sub-optimal. 
The fourth point is a good point, because this point is going to be chosen for many slopes.
The fifth and the sixth are bad points, since they are under the convex hull, and they will never be chosen.
The seventh is a good point again. The eighth is a bad point.
This is exactly what the \rcurve shows. Only the fourth and seventh segments are optimal and match the optimal cost curve.
So, the \rcurve has a segment intersecting with the optimal cost curve for every point on the convex hull. All other segments correspond to sub-optimal decision thresholds.

\section{The area under the \rcurve}\label{area}

If we plug the rate-driven threshold choice method $\Tcostn$ (Eq.\ \ref{eqTn}) into the general formula of the average expected cost for a range of cost proportions (Eq.\ \ref{eqLcost}) we have:

\begin{equation}\label{eqLcostn}
\Lcostn \triangleq \int^1_0{} \Qcost(\Tcostn(c); c) \wcost(c) dc 
\end{equation}
Using the uniform distribution, this expected loss equals the area under the \rcurve. It can be linked to the area under the ROC curve as follows.
\begin{proposition}
\[ \LcostnU =  2 \pi_1\pi_0  (1  - \auc) + \frac{1}{3} - \pi_1\pi_0   \]
\end{proposition}

\begin{proof}
\begin{eqnarray*}
\LcostnU    & = & \int^1_0{} \Qcost(\Tcostn(c); c) U(c) dc \\
            & = &  \int^1_0{} 2 \{c \pi_0 (1 -F_0(\Tcostn(c))) + (1-c)\pi_1 F_1(\Tcostn(c))\}  dc \\
            & = &  \int^1_0{} 2 \{c \pi_0  - c [ \pi_0 F_0(\Tcostn(c)) + \pi_1 F_1(\Tcostn(c))] \} dc + \int^1_0{} 2 \{ \pi_1 F_1(\Tcostn(c))  \} dc 
\end{eqnarray*}
From Lemma \ref{lemmanew} we have that:
\[
\Tcostq[\rate](c) = 
\{ t: F_0(t)\pi_0 + F_1(t)\pi_1 = \rate \}
\]
and of course
\[
\Tcostn(c) = \{ t: F_0(t)\pi_0 + F_1(t)\pi_1 = c \}
\]
Since this is the $t$ which makes the expression equal to $c$ we can find that expression and substitute by $c$. Then we have:
\begin{eqnarray*}
\LcostnU    & = & \int^1_0{} 2 \{c \pi_0  - c ( c ) \} dc + \int^1_0{} 2 \{ \pi_1 F_1(\Tcostn(c))  \} dc \\
            & = & \left[ c^2 \pi_0  - \frac{2 c^3}{3} \right]^1_0 + \int^1_0{} 2 \{ \pi_1 F_1(\Tcostn(c))  \} dc \\						
            & = & \pi_0  - \frac{2}{3}  + 2 \pi_1 \int^1_0{} F_1(\Tcostn(c))  dc \\
\end{eqnarray*}

We have to solve the term $ \int^1_0{} F_1(\Tcostn(c))  dc$.
In order to do this, we have to see that the use of $\Tcostn(c)$ and integrating over $dc$ is like using the mixture distribution for thresholds $t$ and integrating over $dt$. 
%
\begin{eqnarray*}
\int^1_0{} F_1(\Tcostn(c))  dc & = & \int^\infty_{-\infty}{} F_1(t) (\pi_0 f_0(t) + \pi_1 f_1(t)) dt  \\
            & = &  \pi_0 \int^\infty_{-\infty}{} F_1(t) f_0(t) dt + \pi_1 \int^\infty_{-\infty}{} F_1(t) f_1(t) dt \\
						& = &  -\pi_0 \int^\infty_{-\infty}{} (-1 + 1 - F_1(t)) f_0(t) dt + \pi_1 \int^\infty_{-\infty}{} F_1(t) dF_1(t) \\
						& = &  -\pi_0 \int^\infty_{-\infty}{} -1 dt   -\pi_0 \int^\infty_{-\infty}{} (1 - F_1(t)) f_0(t) dt + \pi_1 \int^\infty_{-\infty}{} F_1(t) dF_1(t) \\
						& = &  \pi_0   -\pi_0 \auc + \frac{\pi_1}{2} \\
            & = &  \pi_0  (1  - \auc) + \frac{\pi_1}{2} 
\end{eqnarray*}
And now we can plug this in the expression for the expected cost:
\begin{eqnarray*}
\LcostnU    & = & \pi_0  - \frac{2}{3}  + 2 \pi_1 (\pi_0  (1  - \auc) + \frac{\pi_1}{2} ) \\
            & = & \pi_0  - \frac{2}{3}  + 2 \pi_1\pi_0  (1  - \auc) + \pi_1 \pi_1  \\						
            & = & 2 \pi_1\pi_0  (1  - \auc) + \pi_1 (1- \pi_0)  + \pi_0  - \frac{2}{3}   \\										
            & = & 2 \pi_1\pi_0  (1  - \auc) + 1 - \pi_1\pi_0  - \frac{2}{3}   \\													
            & = & 2 \pi_1\pi_0  (1  - \auc) + \frac{1}{3} - \pi_1\pi_0     
\end{eqnarray*}
\end{proof}
\noindent
This shows that not only has this new curve a clear correspondence to ROC curves, but its area is linearly related to $\auc$.

From costs to skews we have by Lemma \ref{lemma-balance} :
\begin{corollary}
\[ \LsknU =  \frac{1-\auc}{2} + {1 \over 12} \]
\end{corollary}
%
%
%
\noindent
Thus, expected loss is 1/3 for a random classifier, $1/3-1/4=1/12$ for a perfect classifier and $1/3+1/4=7/12$ for the worst possible classifier.

The previous results are obtained for continuous curves with an infinite number of examples.
For empirical curves with a limited number of examples, the result is not exact, but a good approximation.
For instance, for the example in Figure \ref{fig:example2}, we have that $\auc$ is 0.83333.
The area under the \rcurve is 
0.1695 
 for cost proportions, while the theoretical result $2 \pi_1\pi_0  (1  - \auc) + \frac{1}{3} - \pi_1\pi_0 $ gives 0.1701. 
It should be possible to come up with an exact formula for empirical ROC curves; we leave this as an open problem. 

It is interesting to use these general results to get more insight about what the \rcurves mean exactly.
For instance, Figure \ref{fig:example-best-balanced} shows the ROC curve and the \rcurves for a perfect ranker and a balanced data set. 
We used a large number of split points in the ranking to simulate the continuous case. 
We see that our new threshold choice method makes optimal choices for $c=0$, $c=1/2$ and $c=1$ but sub-optimal choices for other operating conditions, which explains the non-zero area under the \rcurve ($1/12$ in the continuous case). 
The optimal choice in this case is to ignore the operating condition altogether and always split the ranking in the middle.

\begin{figure*}
\centering
\includegraphics[width=0.4\textwidth]{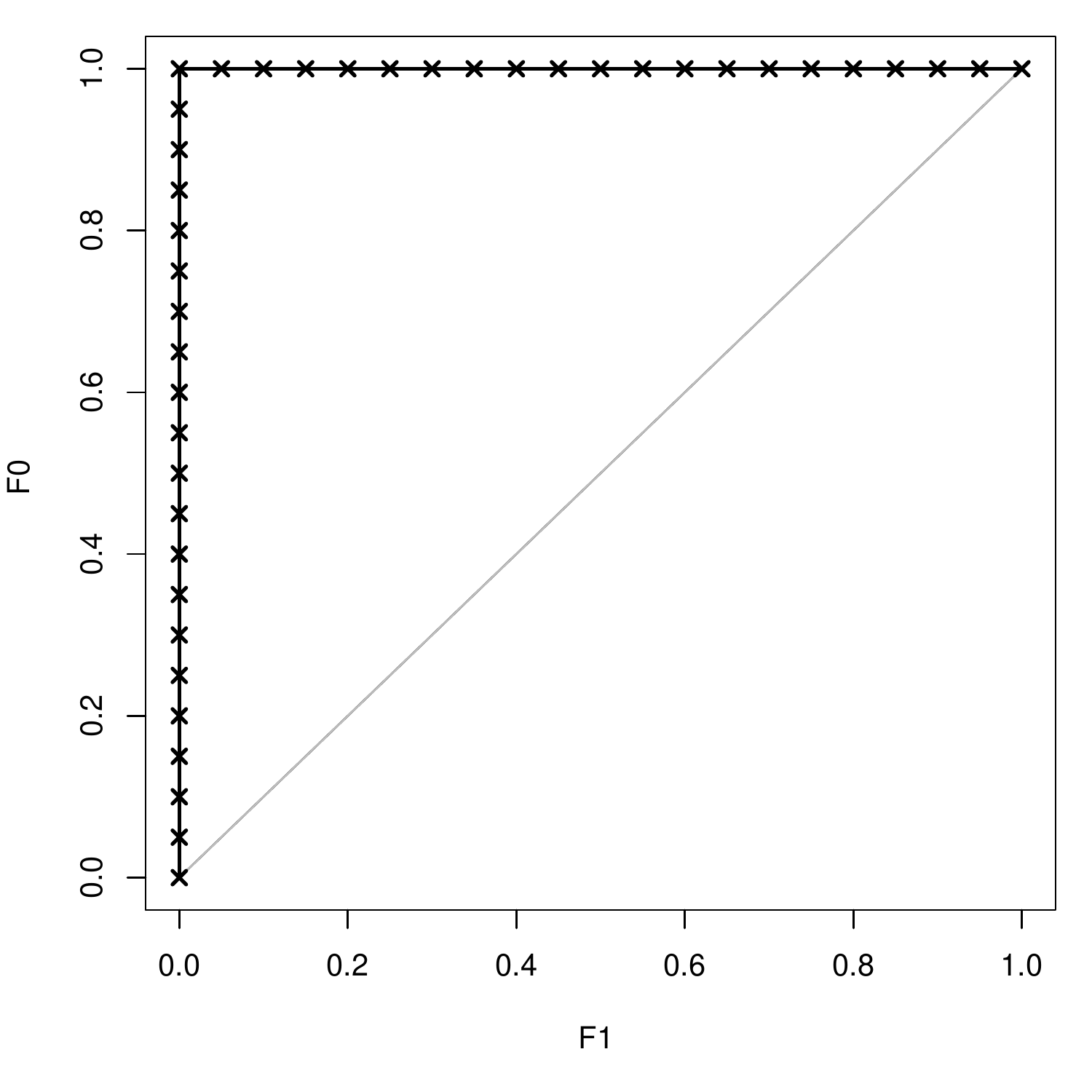}
\includegraphics[width=0.4\textwidth]{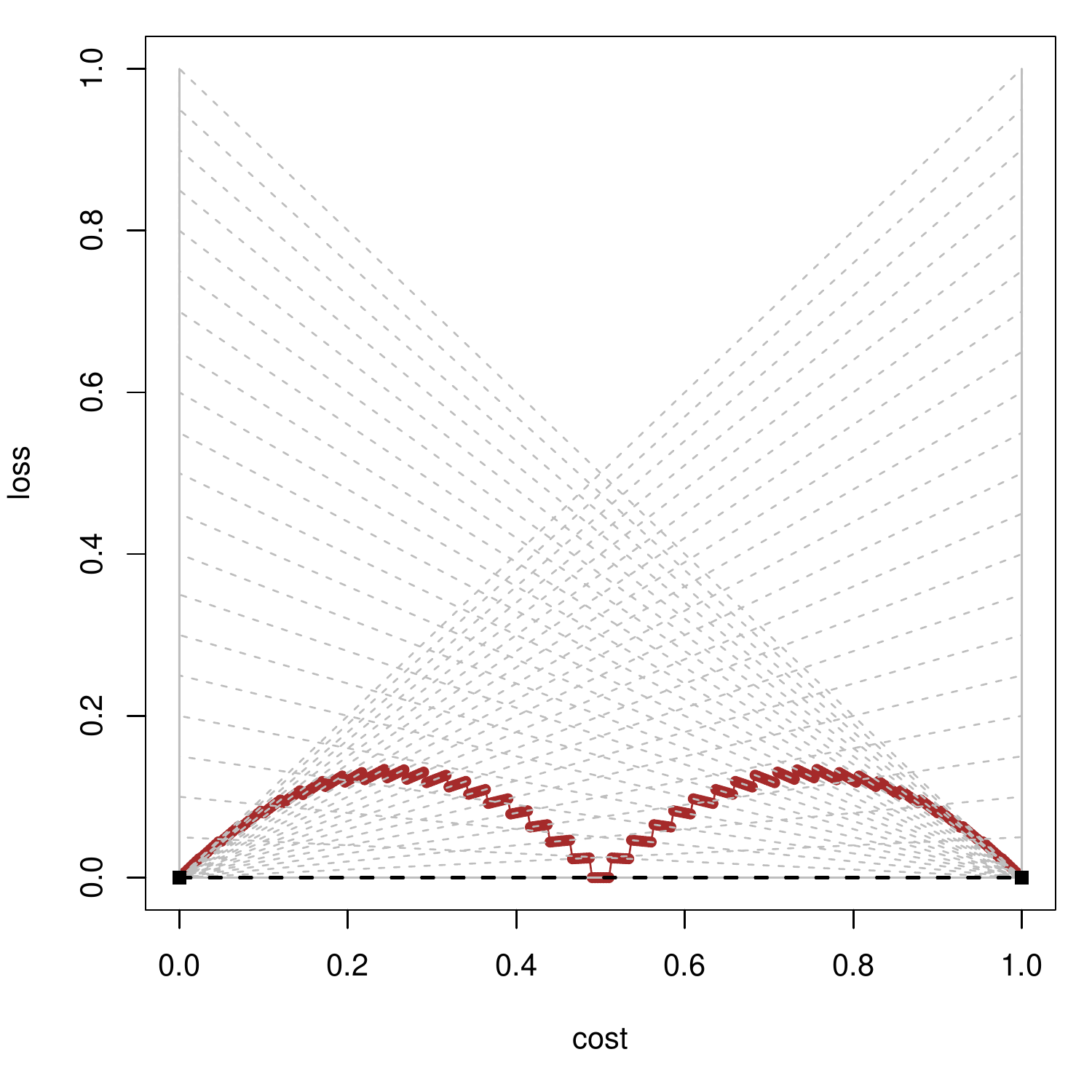}
\caption{Several graphical representations for a perfect and balanced classifier with 20 positive examples and 20 negative examples. 
Left: ROC curve (solid) and convex hull (dashed). Right: cost lines, optimal cost curve (dashed) and \rcurve (thick and solid) against cost proportions. }
\label{fig:example-best-balanced}
\end{figure*}

Figure \ref{fig:example-worst-balanced} shows what the \rcurve looks like for the
worst ranker possible. 
The lower envelope of the cost lines shows that in this case the optimal choice is to always predict 0 if $c<1/2$ and 1 if $c>1/2$, which results in an expected loss of $1/4$. 
In contrast, our new threshold choice method also takes the non-optimal split points into account and hence incurs a higher expected loss ($7/12$ in the continuous case). 
Figure \ref{fig:example-random-balanced} shows what the \rcurve looks like for a classifier which is alternating (close to the diagonal in the ROC space) with $\auc \approx 0.5$. 
Here, the expected loss approximates $4/12= 1/3$, while the optimal choice is the same as in the previous case.
It is not hard to prove that in the limiting case $n \rightarrow \infty$, the ROC cost curve for a random classifier is described by the function $y = 2c (1-c)$, which is the Gini index (the impurity measure, not to be confused with the Gini coefficient which is $2\auc-1$).

\begin{figure*}
\centering
\includegraphics[width=0.4\textwidth]{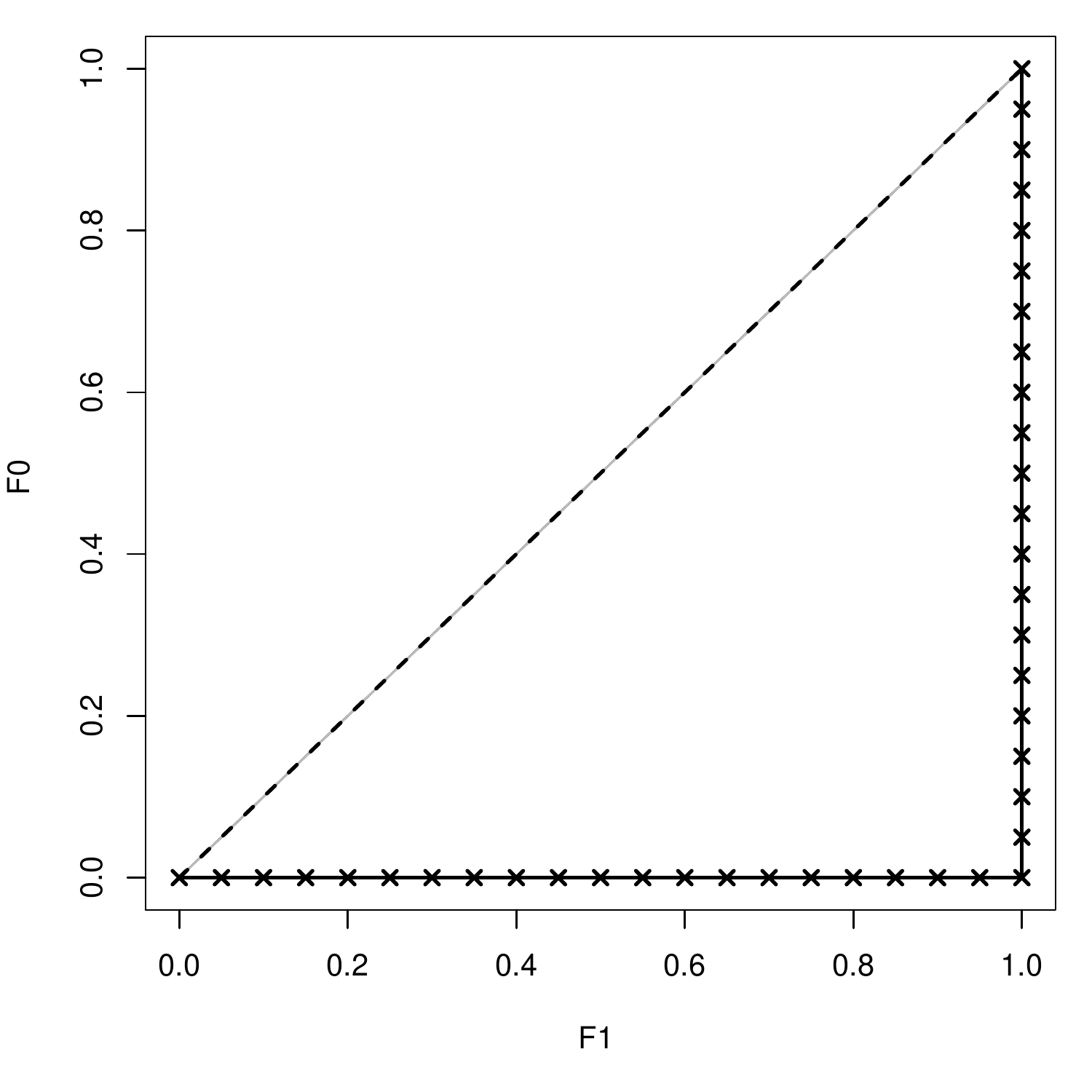}
\includegraphics[width=0.4\textwidth]{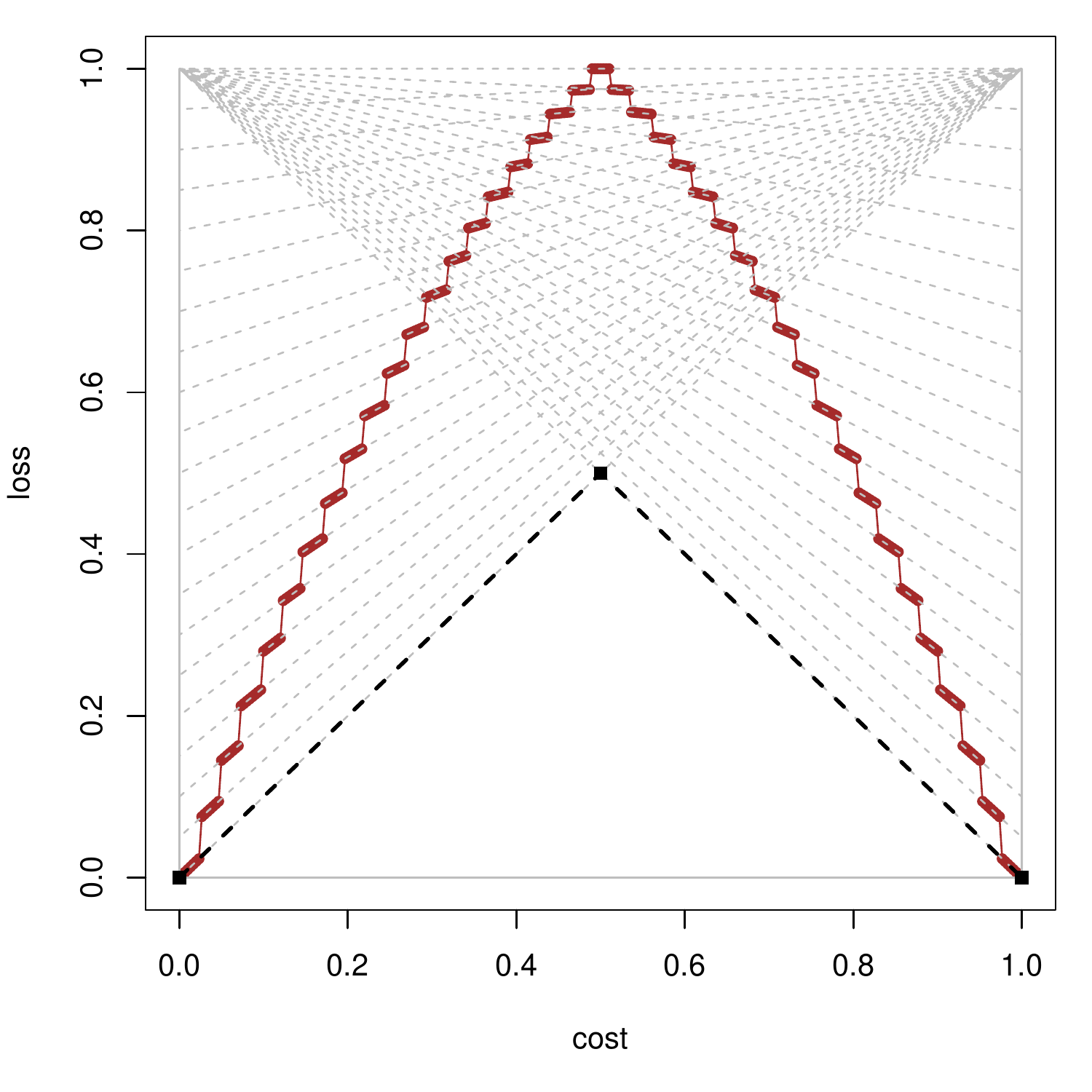}
\caption{Several graphical representations for a very bad classifier (all 0 are ranked before the 1) and balanced classifier with 20 positive examples and 20 negative examples. 
Left: ROC curve (solid) and convex hull (dashed). 
Right: cost lines, optimal cost curve (dashed) and \rcurve (thick and solid) against cost proportions. }
\label{fig:example-worst-balanced}
\end{figure*}

\begin{figure*}
\centering
\includegraphics[width=0.4\textwidth]{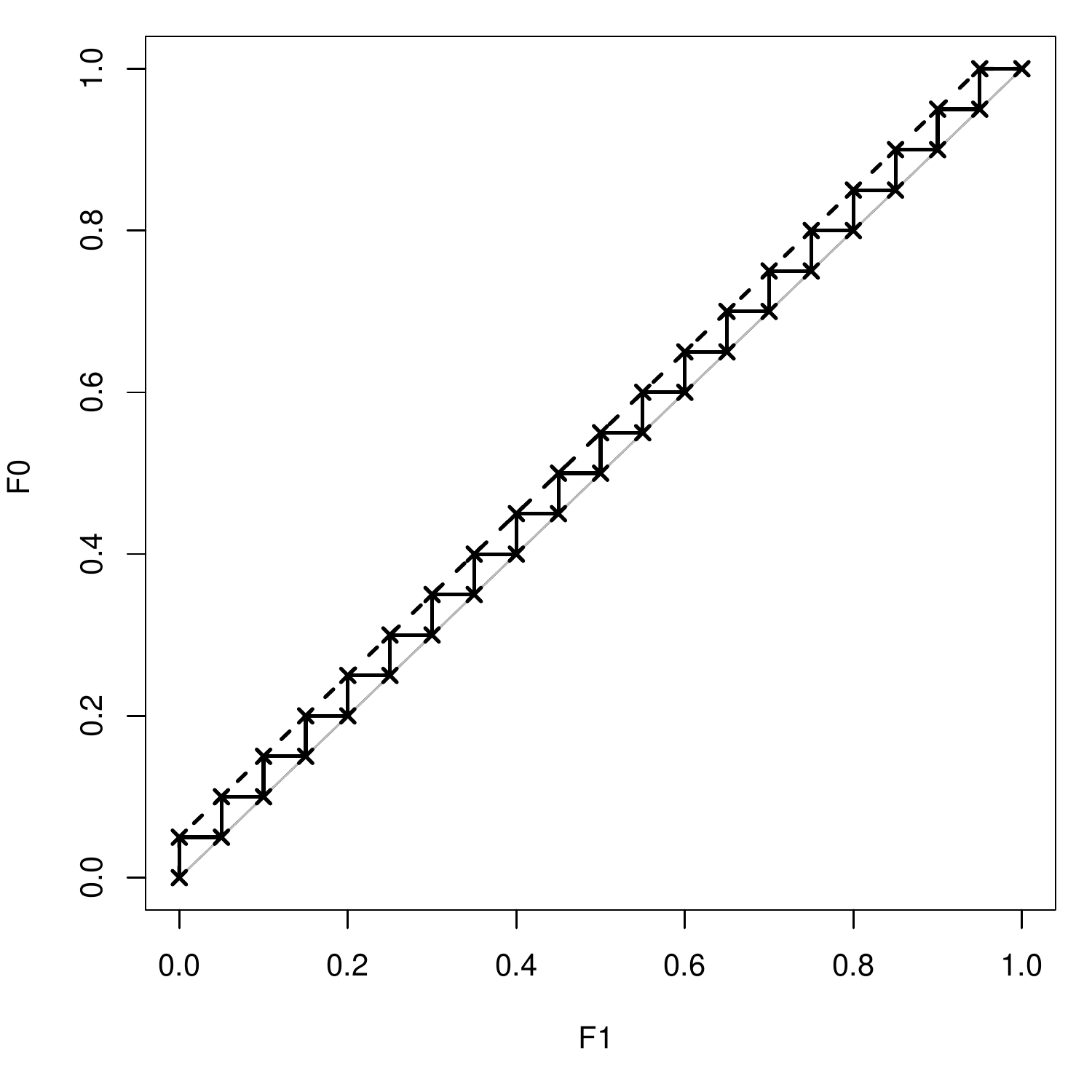}
\includegraphics[width=0.4\textwidth]{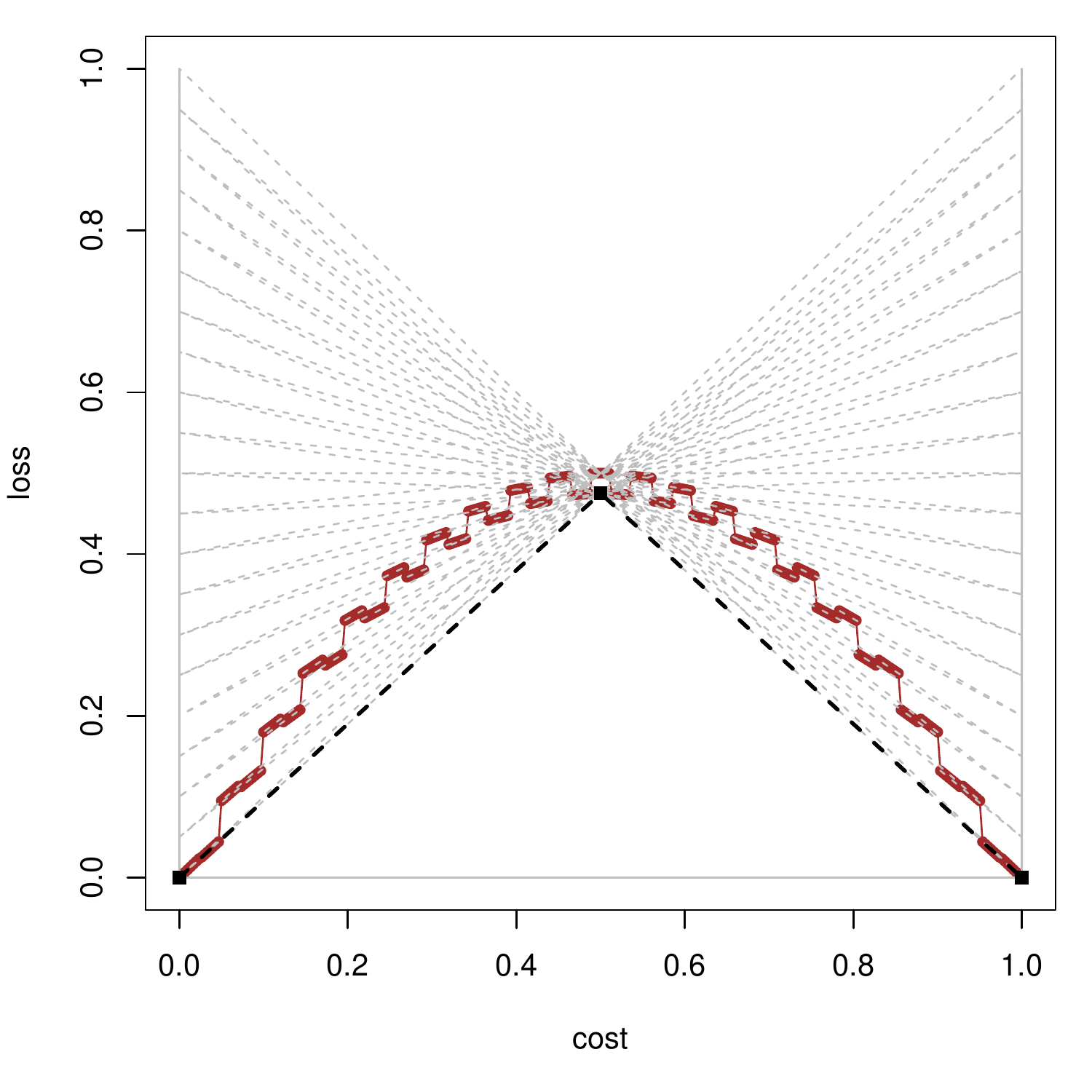}
\caption{Several graphical representations for an alternating (order is 1,0,1,0, ....) and balanced classifier with 20 positive examples and 20 negative examples. 
Left: ROC curve (solid) and convex hull (dashed). 
Right: cost lines, optimal cost curve (dashed) and \rcurve (thick and solid) against cost proportions. }
\label{fig:example-random-balanced}
\end{figure*}

\section{Evenly-spaced scores. The relation between $\auc$ and the Brier score}\label{Brier}

An alternative threshold choice method is to choose $\hat{p}(1|x) = op$ where $op$ is the operating condition. 
This is a natural criterion as it has been used especially when the classifier is a probability estimator. Drummond and Holte \cite{drummond-and-Holte2006} say it is a common example of a ``performance independent criterion''. Referring to Figure 22 in their paper which uses this threshold choice they say: ``the performance independent criterion, in this case, is to set the threshold to correspond to the operating
conditions. For example, if $PC(+)$ = 0.2. the Naive Bayes threshold is set to 0.2". The term $PC(+)$ is equivalent to our `skew'.

Let us see the definition of this method, that we call probabilistic threshold choice (as presented in \cite{ICML11Brier}). We first give the formulation which uses cost proportions for operating conditions:
\begin{equation}\label{eqTcostp}
\Tcostp(c) \triangleq  c
\end{equation}
We define the same thing for $\Tskp(\sk)$:
\begin{equation}\label{eqTskp}
\Tskp(\sk) \triangleq \sk
\end{equation} 
%
If we plug $\Tcostp$ into the general formula of the average expected cost (Eq.\ \ref{eqLcost}) we have the expected probabilistic cost:
\begin{equation}\label{eqLcostp}
\Lcostp \triangleq \int^1_0{} \Qcost(\Tcostp(c); c) \wcost(c) dc =  \int^1_0{} \Qcost(c; c) \wcost(c) dc
\end{equation}
And if we use the uniform distribution and the definition of $\Qcost$ (Eq.\ \ref{eqQcost}):
\begin{eqnarray}\label{eqLcostpU}
\LcostpU & \triangleq & \int^1_0{} \Qcost(c; c) U(c) dc  \nonumber \\
         & =          & \int^1_0{} 2 \{ c \pi_0 (1 -F_0(c)) + (1-c)\pi_1 F_1(c) \} U(c) dc \nonumber \\
  			 & =          & \int^1_0{} 2 \{ c \pi_0 (1 -F_0(c)) \} dc + \int^1_0{} 2 \{ (1-c)\pi_1 F_1(c) \} dc   \label{eqLcostpUlast}
\end{eqnarray}

From here, it is easy to get the following:
\begin{theorem}[\cite{ICML11Brier}]\label{thm:LcostpUequalsBS2}
The expected loss using a uniform distribution for cost proportions is the Brier score. 
\end{theorem}
\begin{proof}
\blue[
We have $\bs = \pi_0 \bs_0 + \pi_1 \bs_1$. 
Using integration by parts, we have 
]
\begin{align*}
\bs_{0} & = \int_{0}^{1} s^{2}f_{0}(s) ds 
  = \left[ s^{2}F_{0}(s) \right]_{s=0}^{1} - \int_{0}^{1} 2sF_{0}(s) ds \\
 & = 1 - \int_{0}^{1} 2sF_{0}(s) ds
  = \int_{0}^{1} 2s ds - \int_{0}^{1} 2sF_{0}(s) ds
\end{align*}
\blue[
Similarly for the negative class:  
]
\begin{align*}
\bs_{1} & = \int_{0}^{1} (1-s)^{2}f_{1}(s) ds \\
 & = \left[ (1-s)^{2}F_{1}(s) \right]_{s=0}^{1} + \int_{0}^{1} 2(1-s)F_{1}(s) ds \\
 & = \int_{0}^{1} 2(1-s)F_{1}(s) ds
\end{align*}
\blue[
Taking their weighted average, we obtain
]
\begin{align*}
\bs & = \pi_0 \bs_0 + \pi_1 \bs_1 \\
 & = \int_{0}^{1} \{\pi_{0}(2s-2sF_{0}(s)) + \pi_{1}2(1-s)F_{1}(s)\} ds
\end{align*}
\blue[
which, after reordering of terms and change of variable, is the same expression as] 
Eq.~(\ref{eqLcostpUlast}).
\end{proof}
\noindent
In \cite{ICML11Brier} we introduced the Brier curve as a plot of $\Qcost(c; c)$ against $c$, so this theorem states that the area under the Brier curve is the Brier score. 

Given a classifier with scores, we may use its scores to try to get better threshold choices with this choice, or we may ignore the scores and use evenly-spaced scores. Namely, we can just assign the $n$ scores such that $s_i = \frac{i-1}{n-1}$, going then from 0 to 1 with steps of $\frac{1}{n-1}$.
With this simple idea we see that the probabilistic threshold choice method reduces to $\Tcostn$, which was analysed in the previous two sections. 
And now we get a very interesting result.
\begin{corollary}\label{corollary_AUC_BS}
If scores are evenly spaced, we get that:
\begin{equation}\label{AUC_BS}
2 \pi_1\pi_0  (1  - \auc) + \frac{1}{3} - \pi_1\pi_0   = \bs
\end{equation}
\end{corollary}
\noindent
As far as we are aware, this is the first published connection between the area under the ROC Curve and the Brier score. 
Of course, this is related to the Brier curves introduced in \cite{ICML11Brier}, so we can also say that \briercurves and \rcurves are closely related (have the same area) if the classifier has evenly-spaced scores.%
\footnote{Working with skews instead of cost proportions, the derivation should lead to a corresponding equation to Corollary \ref{corollary_AUC_BS}, i.e.  $\frac{1}{2}(1  - \auc) + \frac{1}{12}$   = $\frac{\bs_0 + \bs_1}{2}$. This exercise is left to the reader.}

Figure \ref{fig:example-regular-imbalanced} shows a classifier with evenly-spaced scores, so that the previous corollary holds.
We can see that the \briercurve and the \rcurve have similar shapes, although they are not identical. 
We have that the $\auc = 0.7777$, $2 \pi_1\pi_0  (1  - \auc) + \frac{1}{3} - \pi_1\pi_0 = 0.203125$, where the Brier score is 0.2047101 and the area under the Brier curve is 0.2006.\footnote{These two latter numbers should be exactly equal but some small problems when dealing with ties in the implementation of the curves are causing this small difference.}

\begin{figure*}
\centering
\includegraphics[width=1\textwidth]{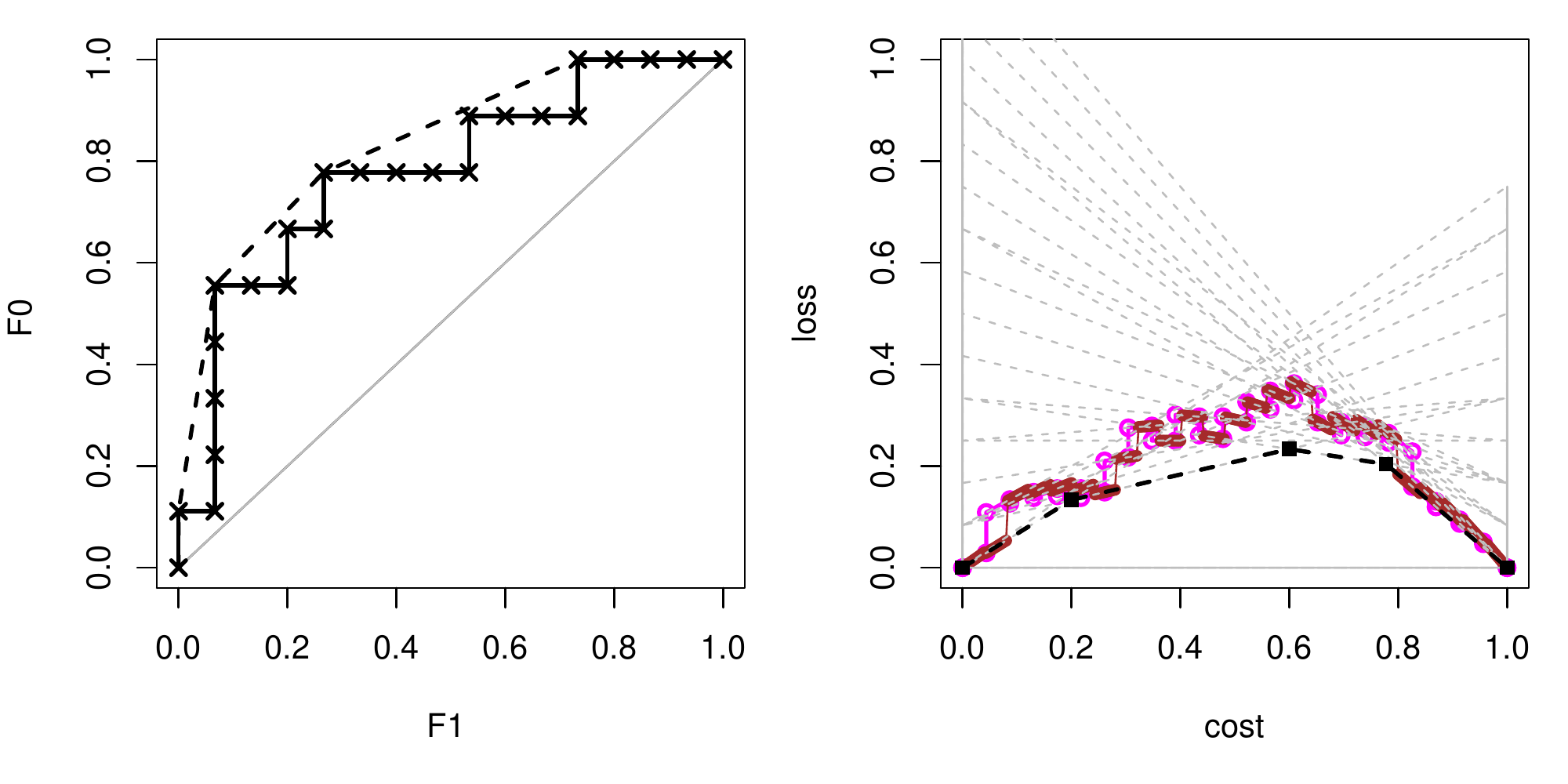}
\caption{Several graphical representations for a ranker with evenly spaced scores 1 0.957 0.913 0.870 0.826 0.782 0.739 0.696 0.652 0.609 0.565 0.522 0.478 0.435 0.391 0.348 0.304 0.261 0.217 0.174 0.130 0.087 0.043 0 and true classes 1,1,1,1,0,1,1,1,0,1,1,1,1,0,1,0,1,1,0,0,0,0,1,0 (15 negative examples and 9 positive examples). 
Left: ROC curve (solid) and convex hull (dashed). Right: cost lines, optimal cost curve (dashed),  \rcurve (brown, thick and solid) and  \briercurve (pink, thin and solid) against cost proportions. 
}
\label{fig:example-regular-imbalanced}
\end{figure*}


Finally, we show a perfectly calibrated classifier and the \rcurves with the \briercurves
in Figure \ref{fig:example-calibrated-imbalanced-with-Brier}. The pink curve (\briercurve) for cost proportions matches the black curve (the optimal curve).
The \rcurve shows that the rate-driven threshold choice methods sometimes makes sub-optimal choices: for example, it only switches to the second point from the left in the ROC curve when $c=4/11=0.36$, whereas the optimal decision would be to switch to this point from $c=0.25$. 

\begin{figure*}
\centering
\includegraphics[width=1\textwidth]{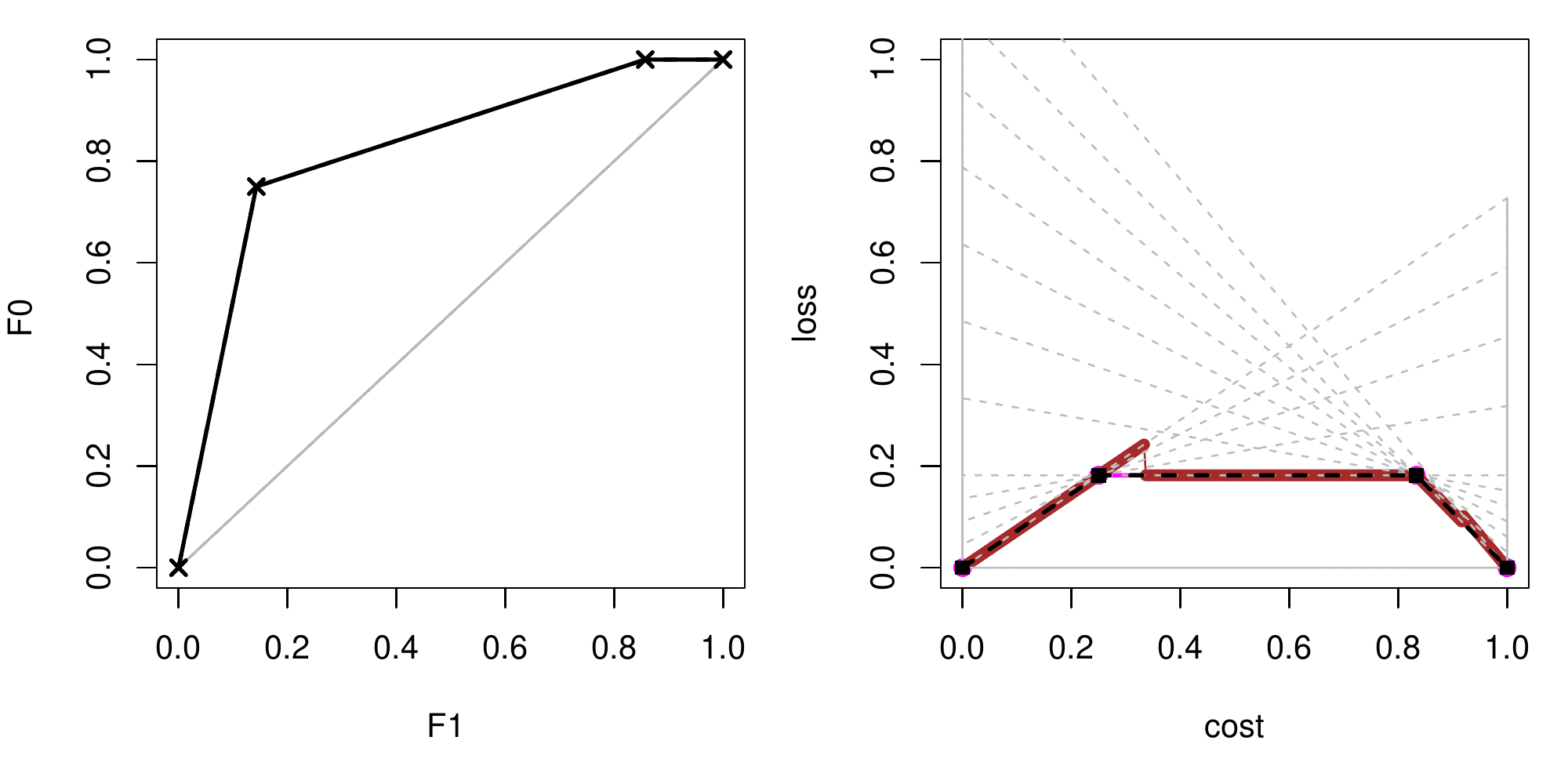}
\caption{Several graphical representations for  a perfectly calibrated classifier with scores
(1,0.833333,0.833333,0.833333,0.833333,0.833333,0.833333, 0.25, 0.25, 0.25, 0.25)
and true classes 
(1,1,1,1,1,1,0,0,0,0,1)
(4 positive examples and 7 negative examples). Left: convex ROC curve. Right: cost lines, optimal cost curve (dashed),  \rcurve (brown, thick and solid) and  \briercurve (pink, thin and solid) against cost proportions. 
}
\label{fig:example-calibrated-imbalanced-with-Brier}
\end{figure*}

\section{Conclusions}\label{conclusion}

The definition of cost curve in the literature has been partially elusive.
While it is clear what cost lines are, it was not clear what different options we may have to draw different curves on the cost space, which of them were valid and which were not, and, more importantly, if they correspond to some curves or representations in ROC space.

In this paper, we have clarified the relation between ROC space and cost space, by finding the corresponding curves for ROC curves in cost space. 
These represent cost curves for rankers that do not commit to a fixed decision threshold. 
Cost plots have some advantages over ROC plots, and the possibility of drawing \rcurves may give further support to use cost plots and use their \rcurves there.

In addition, we have shown that when the scores of a classifier are set in an evenly-spaced way, the \rcurves correspond to the previously presented \briercurves and we have the first firm connection between the Brier Score and $\auc$. This also suggests that there might be a way to draw \briercurves in ROC space.

Given the exploratory character of this paper, there are many interesting options to follow up. 
Our focus will be on how to use \rcurves to choose among models and construct hybrid classifiers.

\bibliographystyle{plain}

\bibliography{biblio}


 \end{document}